\documentclass[twoside,leqno,twocolumn]{article}

\usepackage{ltexpprt}
\usepackage{hyperref}

\usepackage{amsmath}
\usepackage{amssymb}
\usepackage[noabbrev, capitalise]{cleveref}

\usepackage{tikz}
\usepackage{listings}

\newcommand{\R}{{I\!\!R}}

\newcommand{\X}{{\bf x}}
\newcommand{\Z}{{\bf z}}
\newcommand{\Y}{{\bf y}}

\newcommand{\fma}{{\small {\tt fma}}}

\begin{document}

\newcommand\relatedversion{}
\renewcommand\relatedversion{\thanks{An extended version of the paper can be accessed at \protect\url{https://arxiv.org/abs/1902.09310}}} 

\title{\Large Matrix-Free Jacobian Chaining}
\author{Uwe Naumann\thanks{Software and Tools for Computational Engineering, RWTH Aachen University, Germany.}}

\date{}

\maketitle

\fancyfoot[R]{\scriptsize{Copyright \textcopyright\ 20XX by SIAM\\
Unauthorized reproduction of this article is prohibited}}

\begin{abstract}
The efficient computation of Jacobians represents a fundamental challenge in
computational science and engineering. Large-scale modular numerical
simulation programs can be regarded as sequences of evaluations of in our case
differentiable subprograms with corresponding elemental Jacobians. The latter are
typically not available. Tangent and adjoint versions of the
individual subprograms are assumed to be given as results of algorithmic
differentiation 
instead. The classical {\sc (Jacobian) Matrix Chain Product} problem is
reformulated in terms of 
	matrix-free Jacobian-matrix (tangents) and matrix-Jacobian products
	(adjoints), subject to limited memory for storing information required by 
	latter.
	All numerical results can be reproduced using an open-source reference
	implementation.
\end{abstract}

\section{Introduction} \label{sec:intro}

This paper extends our prior work on computational cost-efficient accumulation
of Jacobian matrices. The corresponding combinatorial {\sc Optimal Jacobian
Accumulation} (OJA) problem was shown to be
NP-complete in \cite{Naumann2008OJa}. 
Generalization to higher derivatives is the subject of \cite{Naumann2023Ano}.
Elimination techniques yield different
structural variants of OJA discussed in
\cite{Naumann2004Oao}. Certain special cases turn out to be computationally
tractable as described in \cite{Griewank2003AJa} and \cite{Naumann2008Ove}.

Relevant closely related work by others includes the introduction of
{\sc Vertex Elimination} (VE) \cite{Griewank1991OtC}, an integer programming
approach to VE \cite{Chen2012AIP}, computational experiments with VE
\cite{Forth2004JCG,Tadjouddine2008VoA}, and the formulation of OJA as LU factorization
\cite{Pryce2008FAD}.
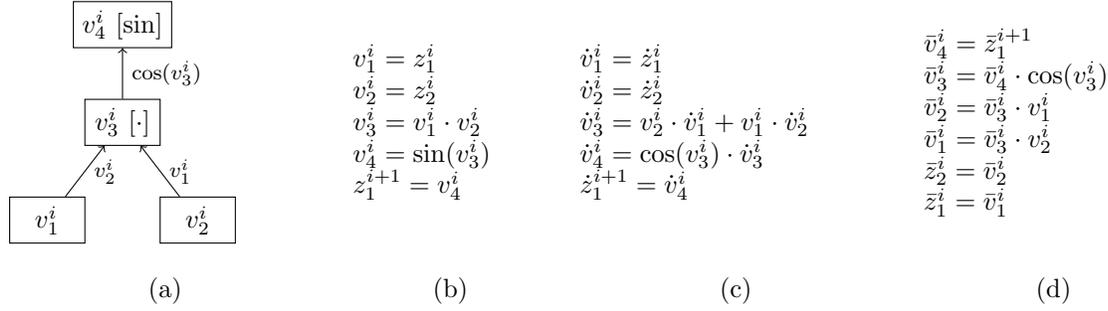
\begin{figure*}[hbt]
\begin{tabular}{cccc}
\begin{minipage}[c]{.24\linewidth}
\begin{tikzpicture}[scale=1, transform shape, rectangle]
  \begin{pgfscope}
\tikzstyle{every node}=[draw,rectangle,minimum width=1cm]
	  \node (0) at (0,0) {$v^i_1$};
	  \node (1) at (2,0) {$v^i_2$};
	  \node (2) at (1,1.3) {$v^i_3~[\cdot]$};
	  \node (3) at (1,2.6) {$v^i_4~[\sin]$};
  \end{pgfscope}
  \begin{scope}[->]
          \draw (0) -- (2) node[midway,right] {\footnotesize $v^i_2$};
          \draw (1) -- (2) node[midway,right] {\footnotesize $v^i_1$};
          \draw (2) -- (3) node[midway,right] {\footnotesize $\cos(v^i_3)$};
  \end{scope}
\end{tikzpicture}

\end{minipage} &
\begin{minipage}[c]{.15\linewidth}
$v_1^i=z_1^i $ \\
$v_2^i=z_2^i $ \\
$v_3^i=v_1^i \cdot v_2^i$ \\
$v_4^i=\sin(v_3^i)$ \\
$z_1^{i+1}=v_4^i$
\end{minipage} &
\begin{minipage}[c]{.24\linewidth}
$\dot{v}_1^i=\dot{z}_{1}^i $ \\
$\dot{v}_2^i=\dot{z}_2^i $ \\
$\dot{v}_3^i=v_2^i \cdot \dot{v}_1^i +v_1^i \cdot \dot{v}_2^i$ \\
$\dot{v}_4^i=\cos(v_3^i) \cdot \dot{v}_3^i$ \\
$\dot{z}_1^{i+1}=\dot{v}_4^i$
\end{minipage} &
\begin{minipage}[c]{.2\linewidth}
$\bar{v}_4^i=\bar{z}_1^{i+1}$ \\
$\bar{v}_3^i=\bar{v}_4^i \cdot \cos(v_3^i)$ \\
$\bar{v}_2^i=\bar{v}_3^i \cdot v_1^i$ \\
$\bar{v}_1^i=\bar{v}_3^i \cdot v_2^i$ \\
$\bar{z}_2^i=\bar{v}_2^i$ \\
$\bar{z}_1^i=\bar{v}_1^i$
\end{minipage} \\
\\
(a) & (b) & (c) & (d)
\end{tabular}
\caption{Simple Example: Labeled DAG (a); primal (b); scalar tangent (c); scalar adjoint (d)} \label{fig:ex1}
\end{figure*}
\begin{figure*}
\begin{equation} \label{eqn:dp1}
\fma_{j,i}=
\begin{cases}
\min(\dot{\fma}_i,\bar{\fma}_i) & j=i \\
\min_{i \leq k < j} \left (\fma_{j,k+1}+\fma_{k,i} + m_j \cdot m_k \cdot n_i \right ) & j>i \; .
\end{cases}
\end{equation}
\end{figure*}

Following a brief discussion of essential prerequisites 
in the remainder of this section, including
 comments on the NP-complete \cite{Garey1979CaI}
{\sc Jacobian Chain Product} problem,
we formulate the combinatorial {\sc Matrix-Free Jacobian Chain Product} 
problem in Section~\ref{sec:GJCP}. NP-completeness of the latter follows
trivially. In Sections~\ref{sec:GDJCPB} and \ref{sec:LMGDJCPB}
we present dynamic programming algorithms for the solution
of two relevant special cases given by the 
{\sc Matrix-Free Dense Jacobian Chain Product Bracketing} 
and 
{\sc Limited-Memory Matrix-Free Dense Jacobian Chain Product Bracketing} 
problems. 
An open-source reference implementation is presented in Section~\ref{sec:impl}.
Numerical results on a number of randomly generated test problems
illustrate the potential benefits to be expected. First steps towards
real-world applications are supported with encouraging results obtained 
for a tunnel 
flow simulation performed with OpenFOAM\footnote{\tt www.openfoam.com} and
described in Section~\ref{sec:cs}. Conclusions are drawn in 
Section~\ref{sec:concl} in the context of additional aspects to be investigated
by ongoing and future research.

Let a multivariate vector function
$$
\Y=F(\X) : \R^n \rightarrow \R^m
$$
(in the following referred to as the {\em primal} function)
be differentiable over the domain of interest, and let
$
F=F_q \circ F_{q-1} \circ \ldots \circ F_2 \circ F_1
$
be such that $$\Z_i=F_i(\Z_{i-1}) : \R^{n_i} \rightarrow \R^{m_i} \; ,$$ for $i=1,\ldots,q$
and $\Z_0=\X,$ $\Y=\Z_q.$
According to the chain rule of differentiation, the Jacobian $F'=F'(\X)$ of $F$ is
equal to
\begin{equation} \label{eqn:jcp}
F' \equiv \frac{d F}{d \X}=F'_q \cdot F'_{q-1} \cdot \ldots \cdot F'_1 \in \R^{m \times n} \; .
\end{equation}
Without loss of generality, we are interested in the minimization of the 
computational cost in terms of {\em fused multiply-add} (\fma) operations 
performed during the evaluation of Equation~(\ref{eqn:jcp}). Weighted variants of this
abstract cost measure allow for translation into practically more relevant
objectives, specifically, run time. 

Algorithmic differentiation (AD) \cite{Griewank2008EDP}
offers two fundamental modes for {\em preaccumulation} of the elemental Jacobians
$$F'_i = F'_i(\Z_{i-1}) \in \R^{m_i \times n_i}$$ prior to the evaluation of the matrix chain product in Equation~(\ref{eqn:jcp}). Directional derivatives are computed in {\em scalar tangent mode} as
\begin{equation} \label{eqn:st}
\dot{\Z}_i=F'_i \cdot \dot{\Z}_{i-1} \in \R^{m_i} \; .
\end{equation}
Accumulation of the entire Jacobian requires evaluation of $n_i$ tangents
in the Cartesian basis directions ${\bf e}_j \in \R^{n_i}$,
if
$F'_i$ is dense. Potential sparsity can and should be detected \cite{Griewank2002DJS} and exploited \cite{Gebremedhin2005WCI,Hossain2005CSJ}.
We denote the computational cost of evaluating a subchain
$F'_j \cdot \ldots \cdot F'_i,$ $j>i,$ of Equation~(\ref{eqn:jcp})
as $\fma_{j,i}$. The computational cost of evaluating $F'_i$ in tangent
mode is denoted as $\fma_{i,i}=\dot{\fma}_i$.

{\em Scalar adjoint mode} yields
\begin{equation} \label{eqn:sa}
\bar{\Z}_{i-1}= \bar{\Z}_i \cdot F'_i \in \R^{1 \times n_i}
\end{equation}
and, hence, dense Jacobians by $m_i$ reevaluations of Equation~(\ref{eqn:sa}) 
with
$\bar{\Z}_i$ ranging over the Cartesian basis directions in $\R^{m_i}.$
The scalar adjoint of $\Z_i$ can be interpreted as the derivative
of some scalar objective with respect to $\Z_i$, yielding
$\bar{\Z}_i \in \R^{1 \times m_i}$ as a row vector.
The computational cost of evaluating $F'_i$ in adjoint
mode is denoted as $\fma_{i,i}=\bar{\fma}_i$.
Further formalization of this cost
estimate will follow in Section~\ref{sec:GJCP}. 

The {\sc [Dense] Jacobian Chain Product Bracketing} problem asks for a
bracketing of the right-hand side of Equation~(\ref{eqn:jcp})
which minimizes the number of \fma\ operations.
It can be solved by dynamic programming
\cite{Bellman1957DP,Godbole1973Oec}
even if the individual factors
are sparse. Sparsity patterns of all subproducts need to be evaluated
symbolically in this case \cite{Griewank2003AJa}.

The dynamic programming recurrence in 
Equation~(\ref{eqn:dp1}) yields an optimal bracketing at a computational cost
of $O(q^3).$
Facilitated by the
{\em overlapping subproblems} and {\em optimal substructure}
properties of {\sc Jacobian Chain Product Bracketing}, the optimization of enclosing chains
looks up tabulated solutions to all subproblems in constant time.

For example, a dense Jacobian chain product of length $q=4$ with
$F'_4 \in \R^4,$ $F'_3 \in \R^{1 \times 4},$ $F'_2 \in \R^{4 \times 5},$ $F'_1 \in \R^{5 \times 3}$ and
$\fma_{4,4}=21,$
$\fma_{3,3}=5,$
$\fma_{2,2}=192,$
$\fma_{1,1}=84$ yields the optimal bracketing
$
F'=F'_4 \cdot ((F'_3 \cdot F'_2) \cdot F'_1)
$
with a cumulative cost of \small $$ (21+5+192+84+1\cdot 4 \cdot 5 + 1 \cdot 5 \cdot 3 + 4 \cdot 1 \cdot 3=)~349 \fma.$$
\normalsize
The more general {\sc Jacobian Chain Product} problem asks for some \fma-optimal
way
to compute $F'$ without the restriction of the search space to valid
bracketings of Equation~(\ref{eqn:jcp}). For example, the matrix product
$$
\begin{pmatrix}
6 & 0 \\
0 & 7 \\
\end{pmatrix}
\begin{pmatrix}
7 & 0 \\
0 & 6 \\
\end{pmatrix}=
\begin{pmatrix}
42 & 0 \\
0 & 42 \\
\end{pmatrix}
$$
can be evaluated at the expense of a single \fma\ as opposed
to two by exploiting commutativity of scalar multiplication.
{\sc Jacobian Chain Product} is known to be NP-complete \cite{Naumann2023Ano}.

\section{\sc Matrix-Free Jacobian Chain Product} \label{sec:GJCP}

The $F_i=F_i(\Z_{i-1})$ induce labeled directed acyclic graphs (DAGs) $G_i=G_i(\Z_{i-1})=(V_i,E_i)$
for $i=1,\ldots,q.$
Vertices in $V_i=\{v^i_j : j=1,\ldots,|V_i|\}$ represent the elemental arithmetic operations $\varphi^i_j \in \{+,\sin,\ldots\}$ executed by the implementation of $F_i$ for given $\Z_{i-1}$. Edges $(j,k) \in E_i \subseteq V_i \times V_i$
mark data dependencies between arguments and results of elemental operations. They are labeled with elemental partial derivatives $$
\frac{\partial \varphi^i_k}{\partial v^i_j} \; , \quad k:~(j,k) \in E_i
$$ of the
elemental functions with respect to their arguments. The concatenation of all 
$G_i$ yields the DAG $G$ of the entire program $F$.
A sample DAG is shown in Figure~\ref{fig:ex1}~(a) corresponding to the primal program in Figure~\ref{fig:ex1}~(b). Note that a single evaluation of the adjoint in
Figure~\ref{fig:ex1}~(d) delivers both gradient entries for $\bar{z}_1^{i+1}=1,$
while two evaluations of the tangent in Figure~\ref{fig:ex1}~(c) with $\dot{\Z}^i=(1~0)^T$ and $\dot{\Z}^i=(0~1)^T$ are required to perform the same task.

Preaccumulation of the elemental Jacobians $F'_i \in \R^{m_i \times n_i}$ requires
either $n_i$ evaluations of the scalar tangent or
$m_i$ evaluations of the scalar adjoint. In order to avoid unnecessary
reevaluation of the function values and of the elemental partial derivatives, we
switch to vectorized versions of tangent and adjoint modes.

For given
$\Z_{i-1} \in \R^{n_i}$ and
$\dot{Z}_{i-1} \in \R^{n_i \times \dot{n}_i}$ the
Jacobian-free evaluation of $$\dot{Z}_i = F_i'(\Z_{i-1}) \cdot \dot{Z}_{i-1} \in \R^{m_i \times \dot{n}_i}$$
in {\em vector tangent mode} is denoted as
\begin{equation} \label{eqn:vt}
\dot{Z}_i := \dot{F}_i(\Z_{i-1}) \cdot \dot{Z}_{i-1} \; .
\end{equation}
Preaccumulation of a dense $F'_i$ requires $\dot{Z}_{i-1}$ to be equal to the
identity $I_{n_i} \in \R^{n_i \times n_i}.$
Equation~(\ref{eqn:vt}) amounts to the simultaneous propagation of $\dot{n}_i$
tangents through $G_i.$ Explicit construction
(and storage) of $G_i$ is not required as the computation of tangents
augments the primal arithmetic locally. For example, the code fragments in
Figure~\ref{fig:ex1}~(b) and (c) can be interleaved as
$v^i_j=\ldots;~\dot{v}^i_j=\ldots$ for
$j=1,\ldots,4.$
Tangent propagation induces a computational cost of
\mbox{$\dot{n}_i \cdot |E_i|$} in addition to the invariant cost of the primal function
evaluation ($|V_i|$) augmented with the computation of all elemental partial
derivatives ($|E_i|$).
The invariant memory requirement of the primal function evaluation is increased by the memory requirement of the tangents,
the minimization of which turns out to be NP-complete as a variant of the
{\sc Directed Bandwidth} problem \cite{Naumann2018LMA}.
In the following, the invariant part of the computational
cost will not be included in cost estimates.

Equation~(\ref{eqn:vt}) can be interpreted as the ``product'' of the DAG $G_i$
with the matrix $\dot{Z}_{i-1}.$
%
%
If $\dot{Z}_{i-1}$ is dense, then its DAG becomes equal to the directed acyclic
version of the complete bipartite graph $K_{n_{i-1},m_{i-1}}.$ The computation
of $\dot{Z}_i$ amounts to the application of the chain rule to the composite
DAG \cite{Baur1983TCo}. It can be interpreted as forward vertex elimination \cite{Griewank1991OtC} yielding a computational cost of \mbox{$\dot{n}_{i-1} \cdot |E_i|$}.

For given
$\Z_{i-1} \in \R^{n_i}$ and
$\bar{Z}_i \in \R^{\bar{m}_i \times m_i},$ the
Jacobian-free evaluation of $$\bar{Z}_{i-1} = \bar{Z}_i \cdot F_i'(\Z_{i-1}) \in \R^{\bar{m}_i \times n_i}$$
in {\em vector adjoint mode}
is denoted as
\begin{equation} \label{eqn:va}
\bar{Z}_{i-1} := \bar{Z}_i \cdot \bar{F}_i(\Z_{i-1}) \; .
\end{equation}
Preaccumulation of a dense $F'_i$ requires $\bar{Z}_i$ to be equal to the
identity
$I_{m_i} \in \R^{m_i \times m_i}.$
Equation~(\ref{eqn:va}) represents the simultaneous backpropagation of
$\bar{m}_i$ adjoints through $G_i.$ 
In the simplest case, $G_i=(V_i,E_i)$ needs to be stored, for example, 
as an ordered list of its edges $E_i$. Consequently, we estimate the 
size of the memory required for storing $G_i$ by $|E_i|.$ 
In AD, the data structure used to represent the DAG is often referred to
as {\em tape}. Hence, we refer to the memory occupied by DAGs as {\em tape memory}.
Real-world estimates
may require scaling of the abstract memory requirement $|E_i|$ depending on details of the given implementation. In the 
following, {\em checkpointing} \cite{Griewank1992ALG,Griewank2000ARA} is assumed to not be
required for computing the elemental Jacobians $F'_i$ in adjoint mode.
For example, the reversal of the data flow requires the adjoint code in
Figure~\ref{fig:ex1}~(d) to be preceded by (the relevant parts of) the
primal code in Figure~\ref{fig:ex1}~(b) (up to and including the computation of $v_3$).
Vector adjoint propagation induces a computational cost of
$\bar{m}_i \cdot |E_i|$ in addition to the invariant cost of a primal function evaluation.
Similar to tangent mode, the minimization of the additional memory required for storing the adjoints
amounts to a variant of the NP-complete {\sc Directed Bandwidth} problem \cite{Naumann2018LMA}.

Equation~(\ref{eqn:va}) can be interpreted as the ``product'' of
the matrix $\bar{Z}_i$ with the DAG $G_i.$
If $\bar{Z}_i$ is dense, then its DAG becomes equal to the directed acyclic
version of the complete bipartite graph $K_{n_i,m_i}.$ The computation
of $\bar{Z}_{i-1}$ amounts to the application of the chain rule to the
composite DAG \cite{Baur1983TCo}. It can be interpreted as backward vertex elimination
\cite{Griewank1991OtC} yielding a computational cost of
\mbox{$\bar{m}_i \cdot |E_i|$}.

Vector tangent and vector adjoint modes belong to the fundamental set of
functionalities offered by the majority of mature AD software tools, for example, \cite{Giering1998RfA,Griewank1996AAC,Hascoet2013TTA,Sagebaum2019HPD}.
Hence, we assume them to be available for all $F_i$ and we refer to them
simply as tangents and adjoints.

Analogous to {\sc Jacobian Chain Product}
the {\sc Matrix-Free Jacobian Chain Product} problem
asks for an algorithm for computing $F'$ with a minimum number of
\fma\ operations for
given tangents and adjoints for all $F_i$ in Equation~(\ref{eqn:jcp}).
As a generalization of an NP-complete problem, {\sc Matrix-Free Jacobian Chain Product} is computationally
intractable too. The corresponding proof of NP-completeness turns out to be very similar to the
arguments presented in \cite{Naumann2023Ano}, and is hence omitted.

\section{\sc Matrix-Free Dense Jacobian Chain Product Bracketing} \label{sec:GDJCPB}

As a special case of {\sc Matrix-Free Jacobian Chain Product}, 
we formulate the {\sc Matrix-Free Dense Jacobian Chain Product Bracketing} 
problem, a solution for which
can be computed efficiently (polynomially in the length of the chain).
A cost-optimal bracketing of the chain rule of differentiation in form of 
a corresponding sequence of applications of tangents and adjoints is sought.
All $F'_i$ are regarded as dense. Optimal exploitation of likely sparsity
would make the problem NP-complete again due to the NP-completeness of
the various variants of {\sc Coloring} for Jacobian compression 
\cite{Gebremedhin2005WCI}; 
see also comments in Section~\ref{sec:concl}.

\begin{figure*}
\begin{equation} \label{eqn:dp}
\fma_{j,i}=
\begin{cases}
	\centering |E_j| \cdot \min\{n_j,m_j\} & j=i \\ \\
	\min_{i \leq k < j} \left \{ \min \left \{
\begin{split}
&\fma_{j,k+1}+\fma_{k,i} + m_j \cdot m_k \cdot n_i, \\
	&\fma_{j,k+1} + m_j \cdot \sum_{\nu=i}^k |E_\nu|, \\
	&\fma_{k,i} + n_i \cdot \sum_{\nu=k+1}^j |E_\nu|
\end{split}
	\right \} \right \} & j>i \; .
\end{cases}
\end{equation}
\end{figure*}
Formally, the {\sc Matrix-Free Dense Jacobian Chain Product Bracketing}
problem is stated as follows:

Let tangents, evaluating $\dot{F}_i \cdot \dot{Z}_i,$ and adjoints,
evaluating
$\bar{Z}_{i+1} \cdot \bar{F}_i,$ be given for all elemental functions
$F_i,$ $i=1,\ldots,q,$ in Equation~(\ref{eqn:jcp}). Treat
all elemental Jacobians as dense.
For a given positive integer $K,$ is there a sequence of evaluations of
the tangents and/or adjoints, such that the number of \fma\ operations
required for the accumulation of the Jacobian $F'$ does not exceed $K$?

\paragraph{Example.}
A matrix-free dense Jacobian chain product of length two yields the following
eight different bracketings:
\begin{alignat*}{2}
	F'&=\dot{F}_2 \cdot F'_1 &&=\dot{F}_2 \cdot (\dot{F}_1  \cdot I_{n_1})
	=\dot{F}_2 \cdot (I_{m_1} \cdot \bar{F}_1) \\
	&=F'_2 \cdot \bar{F}_1&&=(I_{m_2} \cdot \bar{F}_2) \cdot \bar{F}_1
	=(\dot{F}_2 \cdot I_{n_2}) \cdot \bar{F}_1 \\
	&=F'_2 \cdot F'_1&&=(\dot{F}_2 \cdot I_{n_2}) \cdot (I_{m_1} \cdot \bar{F}_1) \\
	&&&=(I_{m_2} \cdot \bar{F}_2) \cdot (I_{m_1} \cdot \bar{F}_1) \\
	&&&=(I_{m_2} \cdot \bar{F}_2) \cdot (\dot{F}_1 \cdot I_{n_1}) \\
	&&&=(\dot{F}_2 \cdot I_{n_2}) \cdot (\dot{F}_1 \cdot I_{n_1}) \; .\\
\end{alignat*}
All bracketings are assumed to be feasible due to availability of
sufficient tape memory for storing the DAGs of all elemental functions 
simultaneously in homogeneous adjoint mode $F'=(I_{m_2} \cdot \bar{F}_2) \cdot \bar{F}_1.$ This, rather strong, assumption will be relaxed in Section~\ref{sec:LMGDJCPB}.

The following theorem presents a dynamic programming algorithm for solving
{\sc Matrix-Free Dense Jacobian Chain Product Bracketing} deterministically 
with a computational cost of $\mathcal{O}(q^3).$

\begin{theorem} \label{the:GDJCPB}
Equation~(\ref{eqn:dp}) solves the {\sc Matrix-Free Dense Jacobian Chain Product Bracketing} problem.
\end{theorem}
\begin{proof}
	We enumerate the four different options from Equation~(\ref{eqn:dp2})
	as
	\begin{itemize}
		\item[(a)] $|E_j| \cdot \min\{n_j,m_j\},$
		\item[(b)] $\min_{i \leq  k < j} \fma_{j,k+1}+\fma_{k,i} + m_j \cdot m_k \cdot n_i,$
		\item[(c)] $\min_{i\leq   k < j} \fma_{j,k+1} + m_j \cdot \sum_{\nu=i}^k |E_\nu|,$ 
		\item[(d)] $\min_{i \leq k <j} \fma_{k,i} + n_i \cdot \sum_{\nu=k+1}^j |E_\nu|.$
			 \end{itemize}
The proof proceeds by induction over $l=j-i.$
\paragraph{Case $0 \leq l \leq 1.$}
All elemental Jacobians $F'_j=F'_{j,j}$
need to be computed in either tangent or adjoint modes at
computational costs of
$n_j \cdot |E_j|$ or
$m_j \cdot |E_j|.$
The respective minima
are tabulated. Special structure of the underlying DAGs $G_i=(V_i,E_i)$ such
as bipartiteness is not exploited. It could result in lower values for
$\fma_{i,i},$ e.g, zero in case of bipartiteness.

The search space for
the product of two dense Jacobians $F'_{i+1} \cdot F'_i$ for given tangents
and adjoints
of $F_{i+1}$ and $F_i$
consists of the following configurations:
\begin{enumerate}
\item $\dot{F}_{i+1} \cdot (\dot{F}_i \cdot I_{n_i}):$ Homogeneous tangent mode
	yields a computational cost of
$$
	\fma_{i+1,i}=n_{i} \cdot |E_{i}|+ n_i \cdot |E_{i+1}| \; .
$$
Equivalently, this scenario can be interpreted as preaccumulation of $F'_i$ in tangent mode followed by evaluation of
$\dot{F}_{i+1} \cdot F'_i.$
This case is covered by Equation~\ref{eqn:dp2}~(a) and (d)
with $n_i \leq m_i.$
\item $\dot{F}_{i+1} \cdot (I_{m_i} \cdot \bar{F}_i):$ Preaccumulation of
$F'_i$ in adjoint mode followed by evaluation of
$\dot{F}_{i+1} \cdot F'_i$ yields a computational cost of
$$
	\fma_{i+1,i}=m_{i} \cdot |E_{i}|+ n_i \cdot |E_{i+1}| \; .
$$
%
%
This case is covered by Equation~\ref{eqn:dp2}~(a) and (d) with
$n_i \geq m_i.$
\item $(I_{m_{i+1}} \cdot \bar{F}_{i+1}) \cdot \bar{F}_i:$
Homogeneous adjoint mode yields a computational cost of
$$
	\fma_{i+1,i}=m_{i+1} \cdot |E_{i+1}|+ m_{i+1} \cdot |E_i| \; .
$$
Equivalently, the preaccumulation of $F'_{i+1}$ in adjoint mode is followed by evaluation of
$F'_{i+1} \cdot \bar{F}_i.$
This case is covered by Equation~\ref{eqn:dp2}~(a) and (c) with
$n_{i+1} \geq m_{i+1}.$
\item $(\dot{F}_{i+1} \cdot I_{n_{i+1}}) \cdot \bar{F}_i:$
Preaccumulation of
		$F'_{i+1}$ in tangent mode followed by evaluation of
$F'_{i+1} \cdot \bar{F}_i$ yields a computational cost of
$$
	\fma_{i+1,i}=n_{i+1} \cdot |E_{i+1}|+ m_{i+1} \cdot |E_i| \; .
$$
%
%
This case is covered by Equation~\ref{eqn:dp2}~(a) and (c) with
$n_{i+1} \leq m_{i+1}.$
\item $(\dot{F}_{i+1} \cdot I_{n_{i+1}}) \cdot (I_{m_i} \cdot \bar{F}_i):$
Preaccumulation of $F'_i$ in adjoint mode followed by
		preaccumulation of $F'_{i+1}$ in tangent mode and
		evaluation of the dense matrix product
$F'_{i+1} \cdot F'_i$ yields a variant of homogeneous preaccumulation with
a computational cost of
$$
	\fma_{i+1,i}=m_{i} \cdot |E_{i}|+n_{i+1} \cdot |E_{i+1}|+m_{i+1} \cdot n_{i+1} \cdot n_i \; .
$$
This case is covered by Equation~\ref{eqn:dp2}~(a) and (b) with
$n_i \geq m_i$ and $n_{i+1} \leq m_{i+1}.$
\end{enumerate}
The remaining three homogeneous preaccumulation options cannot improve the
optimum.
\begin{enumerate}
\item $(\dot{F}_{i+1} \cdot I_{n_{i+1}}) \cdot (\dot{F}_i \cdot I_{n_i}):$
Preaccumulation of both $F'_i$ and $F'_{i+1}$ in tangent mode and
		evaluation of the dense matrix product
$F'_{i+1} \cdot F'_i$ yields a computational cost of
$$
	\fma_{i+1,i}= n_{i} \cdot |E_{i}| +n_{i+1} \cdot |E_{i+1}|+m_{i+1} \cdot n_{i+1} \cdot n_i \; .
$$
It follows that $n_i \leq m_i$ and $n_{i+1} \leq m_{i+1}$ as the
computational cost would otherwise be reduced by preaccumulation
of either $F'_i$ or $F'_{i+1}$ (or both) in adjoint mode.
Superiority of homogeneous
tangent mode follows immediately from
$n_i \leq m_i=n_{i+1} \leq m_{i+1}$ implying
		\begin{align*}
			n_{i} &\cdot |E_{i}|+ n_i \cdot |E_{i+1}| \\ &\leq
			n_{i} \cdot |E_{i}|+ n_{i+1} \cdot |E_{i+1}| \\&<
	n_{i} \cdot |E_{i}| +n_{i+1} \cdot |E_{i+1}|+m_{i+1} \cdot n_{i+1} \cdot n_i \; .
		\end{align*}
\item $(I_{m_{i+1}} \cdot \bar{F}_{i+1}) \cdot (I_{m_i} \cdot \bar{F}_i):$
Preaccumulation of both $F'_i$ and $F'_{i+1}$ in adjoint mode and
evaluation of the dense matrix product
$F'_{i+1} \cdot F'_i$ yields a computational cost of
$$
	\fma_{i+1,i}=m_{i} \cdot |E_{i}|+m_{i+1} \cdot |E_{i+1}|+m_{i+1} \cdot n_{i+1} \cdot n_i \; .
$$
It follows that $n_i \geq m_i$ and $n_{i+1} \geq m_{i+1}$ as the
computational cost would otherwise be reduced by preaccumulation
of either $F'_i$ or $F'_{i+1}$ (or both) in tangent mode.
Superiority of homogeneous adjoint mode follows immediately from
$n_i \geq m_i=n_{i+1} \geq m_{i+1}$ implying
		\begin{align*}
			m_{i+1} &\cdot |E_{i+1}|+ m_{i+1} \cdot |E_i| \\ &\leq
m_{i} \cdot |E_{i}|+m_{i+1} \cdot |E_{i+1}| \\
			&< m_{i} \cdot |E_{i}|+m_{i+1} \cdot |E_{i+1}|+m_{i+1} \cdot n_{i+1} \cdot n_i \; .
		\end{align*}
\item $(I_{m_{i+1}} \cdot \bar{F}_{i+1}) \cdot (\dot{F}_i \cdot I_{n_i}):$
Preaccumulation of $F'_i$ in tangent mode followed by
		preaccumulation of $F'_{i+1}$ in adjoint mode and
		evaluation of the dense matrix product
$F'_{i+1} \cdot F'_i$ yields a computational cost of
$$
	\fma_{i+1,i}=n_{i} \cdot |E_{i}|+m_{i+1} \cdot |E_{i+1}|+m_{i+1} \cdot n_{i+1} \cdot n_i \; .
$$
It follows that $n_i \leq m_i$ and $n_{i+1} \geq m_{i+1}$ as the
computational cost would otherwise be reduced by preaccumulation
of either $F'_i$ in adjoint mode or by preaccumulation of $F'_{i+1}$ in
tangent mode (or both).
This scenario turns out to be inferior to either homogeneous tangent or
adjoint modes. For $n_i \leq m_{i+1}$
\begin{align*}
	n_{i} &\cdot |E_{i}|+ n_i \cdot |E_{i+1}| \\ &\leq
		n_{i} \cdot |E_{i}|+ m_{i+1} \cdot |E_{i+1}| \\ &<
	n_{i} \cdot |E_{i}| +m_{i+1} \cdot |E_{i+1}|+m_{i+1} \cdot n_{i+1} \cdot n_i
\end{align*}
while for
$n_i \geq m_{i+1}$
		\begin{align*}
			m_{i+1} &\cdot |E_{i+1}|+ m_{i+1} \cdot |E_i| \\ &\leq
			m_{i+1} \cdot |E_{i+1}|+ n_i \cdot |E_i| \\ &<
	n_{i} \cdot |E_{i}| +m_{i+1} \cdot |E_{i+1}|+m_{i+1} \cdot n_{i+1} \cdot n_i \; .
\end{align*}
\end{enumerate}
\paragraph{Case $1 \leq l\Rightarrow l+1$.}
{\sc Matrix-Free Dense Jacobian Chain Product Bracketing}
inherits the {\em overlapping subproblems} property from
{\sc Jacobian Chain Product Bracketing}.
It adds two choices
at each split location $i \leq k < j.$
Splitting at position $k$ implies the evaluation of
$F'_{j,i}$ as $F'_{j,k+1} \cdot F'_{k,i}.$
In addition to both
$F'_{j,k+1}$ and $F'_{k,i}$ being available
there are the following two options: $F'_{k,i}$ is available and
it enters the tangent $\dot{F}_{j,k+1} \cdot F'_{k,i}$ as argument;
$F'_{j,k+1}$ is available and it enters the adjoint $F'_{j,k+1} \cdot \bar{F}_{k,i}$ as argument.
All three options yield $F'_{j,i}$ and they correspond to
Equation~\ref{eqn:dp2}~(b)--(d).

The {\em optimal substructure} property remains to be shown. It implies
feasibility of tabulating solutions to all subproblems
for constant-time lookup during the exhaustive search of the $3 \cdot l$
possible scenarios corresponding to the $l$ split locations.

Let the {\em optimal substructure} property not hold for an optimal
$\fma_{l+1,1}$ obtained at split location $1 \leq k < l+1.$ Three cases need
to be distinguished that correspond to Equation~\ref{eqn:dp2}~(b)--(d).
\begin{itemize}
\item[(b)] $\fma_{j,k+1}+\fma_{k,i} + m_j \cdot m_k \cdot n_i:$
The optimal substructure property holds for the preaccumulation of both
$F'_{j,k+1}$ and $F'_{k,i}$ given as chains of length $\leq l.$
The cost of the dense matrix product $F'_{j,k+1} \cdot F'_{k,i}$ is independent
of the respective preaccumulation methods.
For the {\em optimal substructure} property to not hold either the
preaccumulation $F'_{j,k+1}$ or the preaccumulation of $F'_{k,i}$ must be
suboptimal. However, replacement of this suboptimal
preaccumulation method with the tabulated optimum would reduce the overall
cost and hence yield the desired contradiction.
\item[(c)] $\fma_{j,k+1} + m_j \cdot \sum_{\nu=i}^k |E_\nu|:$
The optimal substructure property holds for the preaccumulation of $F'_{j,k+1}.$
The cost of the adjoint $F'_{j,k+1} \cdot \bar{F}_{k,i}$ is independent
of the preaccumulation method.
The replacement of a suboptimal
preaccumulation of $F'_{j,k+1}$ with the tabulated optimum would reduce the
overall cost and hence yield the desired contradiction.
\item[(d)] $\fma_{k,i} + n_i \cdot \sum_{\nu=k+1}^j |E_\nu|:$
The optimal substructure property holds for the preaccumulation of $F'_{k,i}.$
The cost of the tangent $\dot{F}_{j,k+1} \cdot F'_{k,i}$ is independent
of the preaccumulation method.
The replacement of a suboptimal
preaccumulation of $F'_{k,i}$ with the tabulated optimum would reduce the
overall cost and hence yield the desired contradiction.
\end{itemize}
\end{proof}
\paragraph{Example.} We present examples for a
matrix-free dense Jacobian chain product of length two. 
Five configurations
are considered with their solutions corresponding to the five instances of
the search space investigated in the proof of Theorem~\ref{the:GDJCPB}. Optimal
values are highlighted in bold.
\begin{figure*}
\begin{equation} \label{eqn:dp2}
\fma_{j,i}=
\begin{cases}
	\centering |E_j| \cdot 
\begin{cases}
n_j & |E_j| > \overline{M} \\
\min\{n_j,m_j\} & \text{otherwise}
\end{cases}
& j=i \\ \\
	\min_{i \leq k < j} \left \{ \min \left \{
\begin{split}
&\fma_{j,k+1}+\fma_{k,i} + m_j \cdot m_k \cdot n_i, \\
	&\fma_{j,k+1} +
m_j \cdot \sum_{\nu=i}^k |E_\nu| \quad \text{if}~\sum_{\nu=i}^k |E_\nu| \leq \overline{M}, \\
	&\fma_{k,i} + n_i \cdot \sum_{\nu=k+1}^j |E_\nu|
\end{split}
	\right \} \right \} & j>i \; .
\end{cases}
\end{equation}
\end{figure*}
\begin{enumerate}
\item $n_1=2,$ $m_1=n_2=4,$ $m_2=8$, $|E_1|=|E_2|=100:$ 
\begin{align*}
&\fma\left (\dot{F}_2 \cdot (\dot{F}_1  \cdot I_{n_1})\right )={\bf 400} \\
	&\fma\left (\dot{F}_2 \cdot (I_{m_1} \cdot \bar{F}_1) \right )=600 \\
	&\fma\left ((I_{m_2} \cdot \bar{F}_2) \cdot \bar{F}_1 \right )=1600 \displaybreak[0] \\
	&\fma\left ((\dot{F}_2 \cdot I_{n_2}) \cdot \bar{F}_1 \right )=1200 \displaybreak[0] \\
	&\fma\left ((\dot{F}_2 \cdot I_{n_2}) \cdot (I_{m_1} \cdot \bar{F}_1) \right )=864 \displaybreak[0] \\
	&\fma\left ((I_{m_2} \cdot \bar{F}_2) \cdot (I_{m_1} \cdot \bar{F}_1) \right )=1264 \displaybreak[0] \\
	&\fma\left ((I_{m_2} \cdot \bar{F}_2) \cdot (\dot{F}_1 \cdot I_{n_1}) \right )=1064 \displaybreak[0] \\
	&\fma\left ((\dot{F}_2 \cdot I_{n_2}) \cdot (\dot{F}_1 \cdot I_{n_1}) \right )= 664.
\end{align*}
\item $n_1=4,$ $m_1=n_2=2,$ $m_2=32$, $|E_1|=|E_2|=100:$
\begin{align*}
&\fma\left (\dot{F}_2 \cdot (\dot{F}_1  \cdot I_{n_1})\right )=800 \\
&\fma\left (\dot{F}_2 \cdot (I_{m_1} \cdot \bar{F}_1) \right )={\bf 600} \\
&\fma\left ((I_{m_2} \cdot \bar{F}_2) \cdot \bar{F}_1 \right )=6400 \\
&\fma\left ((\dot{F}_2 \cdot I_{n_2}) \cdot \bar{F}_1 \right )=3400 \\
&\fma\left ((\dot{F}_2 \cdot I_{n_2}) \cdot (I_{m_1} \cdot \bar{F}_1) \right )=656 \\
&\fma\left ((I_{m_2} \cdot \bar{F}_2) \cdot (I_{m_1} \cdot \bar{F}_1) \right )=3656\\
&\fma\left ((I_{m_2} \cdot \bar{F}_2) \cdot (\dot{F}_1 \cdot I_{n_1}) \right )=3856 \\
&\fma\left ((\dot{F}_2 \cdot I_{n_2}) \cdot (\dot{F}_1 \cdot I_{n_1}) \right )= 856.
\end{align*}
\item $n_1=8,$ $m_1=n_2=4,$ $m_2=2$, $|E_1|=|E_2|=100:$
\begin{align*}
&\fma\left (\dot{F}_2 \cdot (\dot{F}_1  \cdot I_{n_1})\right )=1600 \displaybreak[0] \\
&\fma\left (\dot{F}_2 \cdot (I_{m_1} \cdot \bar{F}_1) \right )=1200 \displaybreak[0] \\
&\fma\left ((I_{m_2} \cdot \bar{F}_2) \cdot \bar{F}_1 \right )={\bf 400} \displaybreak[0] \\
&\fma\left ((\dot{F}_2 \cdot I_{n_2}) \cdot \bar{F}_1 \right )=600 \displaybreak[0] \\
&\fma\left ((\dot{F}_2 \cdot I_{n_2}) \cdot (I_{m_1} \cdot \bar{F}_1) \right )=864 \displaybreak[0] \\
&\fma\left ((I_{m_2} \cdot \bar{F}_2) \cdot (I_{m_1} \cdot \bar{F}_1) \right )=664 \displaybreak[0] \\
&\fma\left ((I_{m_2} \cdot \bar{F}_2) \cdot (\dot{F}_1 \cdot I_{n_1}) \right )=1064 \displaybreak[0] \\
&\fma\left ((\dot{F}_2 \cdot I_{n_2}) \cdot (\dot{F}_1 \cdot I_{n_1}) \right )= 1264.
\end{align*}
\item $n_1=32,$ $m_1=n_2=2,$ $m_2=4$, $|E_1|=|E_2|=100:$
\begin{align*}
&\fma\left (\dot{F}_2 \cdot (\dot{F}_1  \cdot I_{n_1})\right )=6400 \\
&\fma\left (\dot{F}_2 \cdot (I_{m_1} \cdot \bar{F}_1) \right )=3400 \\
&\fma\left ((I_{m_2} \cdot \bar{F}_2) \cdot \bar{F}_1 \right )=800 \\
&\fma\left ((\dot{F}_2 \cdot I_{n_2}) \cdot \bar{F}_1 \right )={\bf 600} \\
&\fma\left ((\dot{F}_2 \cdot I_{n_2}) \cdot (I_{m_1} \cdot \bar{F}_1) \right )=656 \\
&\fma\left ((I_{m_2} \cdot \bar{F}_2) \cdot (I_{m_1} \cdot \bar{F}_1) \right )=856 \\
&\fma\left ((I_{m_2} \cdot \bar{F}_2) \cdot (\dot{F}_1 \cdot I_{n_1}) \right )=3856 \\
&\fma\left ((\dot{F}_2 \cdot I_{n_2}) \cdot (\dot{F}_1 \cdot I_{n_1}) \right )= 3656.
\end{align*}
\item $n_1=4,$ $m_1=n_2=2,$ $m_2=4$, $|E_1|=|E_2|=100:$
\begin{align*}
&\fma\left (\dot{F}_2 \cdot (\dot{F}_1  \cdot I_{n_1})\right )=800 \\
&\fma\left (\dot{F}_2 \cdot (I_{m_1} \cdot \bar{F}_1) \right )=600 \\
&\fma\left ((I_{m_2} \cdot \bar{F}_2) \cdot \bar{F}_1 \right )=800 \\
&\fma\left ((\dot{F}_2 \cdot I_{n_2}) \cdot \bar{F}_1 \right )=600 \\
&\fma\left ((\dot{F}_2 \cdot I_{n_2}) \cdot (I_{m_1} \cdot \bar{F}_1) \right )={\bf 432} \\
&\fma\left ((I_{m_2} \cdot \bar{F}_2) \cdot (I_{m_1} \cdot \bar{F}_1) \right )=632 \\
&\fma\left ((I_{m_2} \cdot \bar{F}_2) \cdot (\dot{F}_1 \cdot I_{n_1}) \right )=832 \\
&\fma\left ((\dot{F}_2 \cdot I_{n_2}) \cdot (\dot{F}_1 \cdot I_{n_1}) \right )= 632.
\end{align*}
\end{enumerate}

\section{\sc Limited-Memory Matrix-Free Dense Jaco- bian Chain Product Bracketing} \label{sec:LMGDJCPB}

Real-world scenarios have to deal with limited tape memory. 
Homogeneous adjoint mode may not be feasible, even if
it is restricted to subchains of Equation~(\ref{eqn:jcp}). 

As a special case of {\sc Matrix-Free Dense Jacobian Chain Product Bracketing} 
we formulate the {\sc Limited-Memory Matrix-Free Dense Jacobian Chain Product
Bracketing} 
problem, a solution for which
can still be computed efficiently.
As before, all $F'_i$ are regarded as dense. 
A cost-optimal feasible sequence of applications of tangents and adjoints is 
sought.
Feasibility is defined as the size of tape memory not 
exceeding a given upper bound $\overline{M}.$ Argument checkpoints $\Z_{i-1}$ 
are assumed to be available for all $F_i$ allowing for individual DAGs to be
recorded as necessary. A single additional evaluation of $F$ is required.
Moreover, the argument checkpoints are assumed to not count towards the 
tape memory, which turns out to be a reasonable simplification
if their sizes are negligible relative to the sizes of 
the corresponding DAGs.

Formally, the {\sc Limited-Memory Matrix-Free Dense Jacobian Chain Product 
Bracketing} problem is stated as follows:

Let tangents, evaluating $\dot{F}_i \cdot \dot{Z}_i,$ and adjoint, evaluating 
$\bar{Z}_{i+1} \cdot \bar{F}_i,$ be given for all elemental functions
$F_i,$ $i=1,\ldots,q,$ in Equation~(\ref{eqn:jcp}). Treat
all elemental Jacobians as dense.
For a given positive integer $K,$ is there a feasible sequence of evaluations of
the tangents and/or adjoints, such that the number of \fma\ operations
required for the accumulation of the Jacobian $F'$ does not exceed K?
Feasibility is defined as the size of tape memory required by the longest 
homogeneous adjoint subchain of a candidate solution not 
exceeding a given upper bound $\overline{M} \geq 0.$ 

\begin{theorem} \label{the:LMGDJCPB}
Equation~(\ref{eqn:dp2}) solves the {\sc Limited-Memory Matrix-Free Dense Jacobian Chain Product Bracketing} problem.
\end{theorem}
\begin{proof}
A subspace of the search space of 
{\sc Matrix-Free Dense Jacobian Chain Product Bracketing} needs to be explored.
Infeasible adjoints are simply discarded. Both the 
{\em overlapping subproblems} and {\em optimal substructure} properties of dynamic programming are preserved. 
Optimal feasible solutions to all subproblems can be tabulated for 
lookup in constant time.
\end{proof}

\paragraph{Example.} To illustrate the impact of limited memory,
let $F=F_3 \circ F_2 \circ F_1 : \R^8 \rightarrow \R,$ such that
\begin{align*}
F_1 &: \R^8 \rightarrow \R^4,~|E_1|=32  \\
F_2 &: \R^4 \rightarrow \R^2,~|E_2|=16 \\
F_3 &: \R^2 \rightarrow \R,~|E_3|=8 \; .
\end{align*}
Homogeneous adjoint mode outperforms all other Jacobian accumulations
in unlimited memory, actually using 
$M=|E_1|+|E_2|+|E_3|=56.$
The following table illustrates the impact of a decreasing upper bound
$\overline{M}$ for the tape memory on both structure and
computational cost of corresponding feasible Jacobian accumulations given
Equation~(\ref{eqn:dp2}).
\begin{center}
\begin{tabular}{|c|c|c|c|}
\hline
&&&\vspace{-3mm} \\
$\overline{M}$ & $F'=$ & $\fma$ & $M$ \\
&&&\vspace{-3mm} \\
\hline 
&&&\vspace{-3mm} \\
56 & $1 \cdot \bar{F}_3 \cdot \bar{F}_2 \cdot \bar{F}_1$ & 56 & 56 \\
55 & $(1 \cdot \bar{F}_3) \cdot (I_2 \cdot \bar{F}_2 \cdot \bar{F}_1)$ & 120 & 48 \\
47 & $(1 \cdot \bar{F}_3 \cdot \bar{F}_2) \cdot  (I_4 \cdot \bar{F}_1)$ & 184 & 32 \\
31 & $(1 \cdot \bar{F}_3 \cdot \bar{F}_2) \cdot  (\dot{F}_1 \cdot I_8)$ & 312 & 24 \\
23 & $(1 \cdot \bar{F}_3) \cdot (I_2 \cdot \bar{F}_2) \cdot  (\dot{F}_1 \cdot I_8)$ & 336 & 16 \\
15 & $(1 \cdot \bar{F}_3) \cdot (\dot{F}_2 \cdot I_4) \cdot  (\dot{F}_1 \cdot I_8)$ & 368 & 8 \\
7 & $(\dot{F}_3 \cdot I_2) \cdot (\dot{F}_2 \cdot I_4) \cdot  (\dot{F}_1 \cdot I_8)$ & 376 & 0 \\
\hline
\end{tabular}
\end{center}
$\overline{M}$ is set to undercut the tape memory requirement $M$
of the preceding
scenario by one unit until the latter vanishes
identically. The resulting feasible Jacobian accumulations 
relax adjoint mode gradually by introducing local preaccumulation, first
in adjoint (second and third rows) and later in (combination with) tangent 
mode (from fourth row).
The computational cost grows accordingly.

\section{Implementation and Results} \label{sec:impl}

Our reference implementation can be downloaded from
\begin{center}
	\tt
	www.github.com/un110076/ADMission/MFJC
\end{center}
together with all sample problems referred to in this paper.
All results can thus be reproduced.
Two executables are provided:
\verb!generate.exe! generates problem
instances randomly for a given length $q$ of the chain and upper
bound on the number of rows and columns of the individual
factors. Two additional command line parameters specify lower and upper bounds
for the number of edges in the individual DAGs.
The output can be redirected into a text file which serves as input
to \verb!solve.exe!.
The latter computes one optimal feasible solution to the given
problem instance for given $\overline{M}$. This solution is compared with the 
costs of the
homogeneous tangent, adjoint and preaccumulation methods. The latter
represents a solution of the resulting
{\sc Dense Jacobian Chain Product Bracketing} problem.

\begin{table*}
\centering
\begin{tabular}{|c|c|c|c|c|c|}
\hline
\verb!len! & Tangent & Adjoint & Preaccumulation & Optimum & \\
\hline
	10 &  3,708 & 5,562 & \bf 2,618 & \bf 1,344 & 2 \\
	50 & \bf 1,283,868 & 1,355,194 & 1,687,575 & \bf 71,668 & 18\\
	100 &  \bf 3,677,565 & 44,866,293 & 40,880,996 & \bf 1,471,636 & 2 \\
	250 &  \bf 585,023,794 & 1,496,126,424 & 1,196,618,622& \bf 9,600,070 & 61 \\
	500 & 21,306,718,862 & 19,518,742,454 & \bf 18,207,565,409 & \bf 149,147,898 & 122 \\
\hline
\end{tabular}
	\caption{Test Results: Cost in $\fma;$ Superiority of the solutions to {\sc Matrix-Free Dense Jacobian Chain Bracketing} are quantified as the rounded ratios of the highlighted entries in each row; see last column.} \label{tab:res}
\end{table*}

\paragraph{Example.} Run \verb!generate.exe 3 3 0 40! to get, for example,
the problem instance
$
F'=F'_3 \cdot F'_2 \cdot F'_1,
$
where \\
\\
$F'_1 \in \R^{3 \times 3} \rightarrow G_1=(V_1,E_1):~|E_1|=29 $\\
$F'_2 \in \R^{1 \times 3} \rightarrow G_2=(V_2,E_2):~|E_2|=14 $\\
$F'_3 \in \R^{2 \times 1} \rightarrow G_3=(V_3,E_3):~|E_3|=7 \; .$ \\
\\
Let this problem description be stored in the text file \verb!problem.txt!.
Note that the conservative $\overline{M}=50$ yields an instance of the 
{\sc Matrix-Free Dense Jacobian Chain Bracketing} problem as homogeneous
adjoint mode would result in a tape memory requirement of $M=50.$
Run 
\verb!solve.exe problem.txt 50!
to get the following output:

\begin{lstlisting}[basicstyle=\footnotesize]

G_{1}=[ 3 3 29 ]
G_{2}=[ 1 3 14 ]
G_{3}=[ 2 1 7 ]

Dynamic Programming Table:
F'_{1,1}: GxM(0); fma=87; M=0;
F'_{2,2}: MxG(0); fma=14; M=14;
F'_{2,1}: MxG(1); fma=43; M=43;
F'_{3,3}: GxM(0); fma=7; M=0;
F'_{3,2}: MxM(2); fma=27; M=14;
F'_{3,1}: MxM(2); fma=56; M=43;

Optimal Cost=56
Memory Requirement=43

Cost of
  homogeneous tangent mode=150
  homogeneous adjoint mode=100
  optimal preaccumulation=108+15=123
\end{lstlisting}
$F'_1$ is optimally accumulated in tangent mode, {\footnotesize \lstinline{GxM(0)}}, at the
expense of $3\cdot 29=87 \fma$ without tape (similarly,
$F'_2$ in adjoint mode, {\footnotesize \lstinline{MxG(0)}}, at $1\cdot 14=14 \fma$ with tape of size 14, and
$F'_3$, again, in tangent mode at $1\cdot 7=7 \fma$ without tape). 
The argument denotes the split position, which turns out
obsolete in case of preaccumulation of elemental Jacobians, denoted 
by {\footnotesize \lstinline{0}}.
The optimal method to compute
$F'_{2,1}$ uses adjoint mode as $F'_2 \cdot \bar{F}_1$ at cost
$14+1 \cdot 29=43 \fma.$
Preaccumulation of $F'_2$ and $F'_3$ followed by the dense matrix product
$F'_3 \cdot F'_2$, denoted by {\footnotesize \lstinline{MxM(2)}} turns out to 
be the optimal method for computing $F'_{3,2}.$
The entire problem instance is evaluated optimally as
$$
F'=(\dot{F}_3 \cdot I_1) \cdot ((I_1 \cdot \bar{F}_2) \cdot \bar{F}_1) \; ,
$$
which corresponds to the fourth, third and sixth entries in the {\footnotesize \lstinline{Dynamic Programming Table}}.
The computational cost becomes equal to $7\cdot 1 + (14+29)\cdot 1 + 2 \cdot 1 \cdot 3=56 \fma.$ The adjoint subchain $(I_1 \cdot \bar{F}_2) \cdot \bar{F}_1$ 
(third table entry) induces
the maximum tape memory requirement of $43=29+14$.

Homogeneous tangent mode
$$
F':=\dot{F}_3 \cdot (\dot{F}_2 \cdot (\dot{F}_1 \cdot I_{n_1}))
$$
yields a cost of $n_1 \cdot \sum_{i=1}^3 |E_i|=3 \cdot (29+14+7)=150 \fma.$
Homogeneous adjoint mode
$$
F':=((I_{m_3} \cdot \bar{F}_3) \cdot \bar{F}_2) \cdot \bar{F}_1
$$
yields a cost of $m_3 \cdot \sum_{i=1}^3 |E_i|=2 \cdot (29+14+7)=100 \fma.$
It becomes infeasible for 
$\overline{M}<50.$
Optimal preaccumulation of $F'_i$ for $i=1,2,3$ takes
$\sum_{i=1}^3 |E_i| \cdot \min(m_i,n_i)=1 \cdot 7 + 1 \cdot 14 + 3 \cdot 29=108 \fma$ followed by optimal bracketing as
$
F'=F'_3 \cdot (F'_2 \cdot F'_1),
$
adding $9+6=15 \fma$ and
yielding a total cost of the optimal homogeneous preaccumulation method of $108+15=123 \fma.$

The dynamic programming solution of the
{\sc Matrix-Free Dense Jacobian Chain Product Bracketing} problem
yields an improvement of nearly $50$ percent over homogeneous adjoint mode.
Less conservative tape memory bounds increase the cost of Jacobian 
accumulation, for example, $\overline{M}=42 \Rightarrow \fma=123,~M=14$ and
	$\overline{M}=13 \Rightarrow \fma=142,~M=0.$
Readers are encouraged to reproduce these results using the given reference 
implementation.

In Table~\ref{tab:res} we present results for random problem instances of
growing sizes. 
The dynamic programming solutions can improve the best homogeneous
method by more than two order of magnitude. Full specifications of all five test
problems can be found in the github repository.
Readers are encouraged to investigate the impact of limited memory on the
respective solutions.

\section{Case Study} \label{sec:cs}

When tackling real-world problems, a partitioning of the differentiable program 
$\Y=F(\X)$ into elemental functions may not be obvious.
Automatic detection of a suitable partitioning of $F$ is highly desirable.
This is equivalent to finding (minimal) vertex cuts in the DAG 
implied by an execution of the primal for some input $\X$. 

A detailed discussion of the
corresponding dynamic program analysis based on operator and function 
overloading as employed by various software tools for AD \cite{Griewank1996AAC,Hogan2014FRM,Sagebaum2019HPD} is omitted due to space restrictions.
It follows closely the principles of activity analysis outlined
for the static case in \cite{Hascoet2005TBR}. 
Vertex cuts in the DAG
represent the set of active variables at the given point in the underlying
simulation. Split positions are placed at selected local minima. 
A balance of the computational costs of the resulting elemental functions
is sought heuristically.

To showcase the applicability of our approach to real-world scenarios, we 
generate the DAG for part of a tunnel flow simulation performed with the 
open-source CFD library OpenFOAM.
A plot of the evolution of the vertex cut size can be seen in Figure~\ref{fig:foam_ve}. The entire simulation amounts to a chain of this pattern. 

\begin{figure}
  \includegraphics[width=.96\columnwidth]{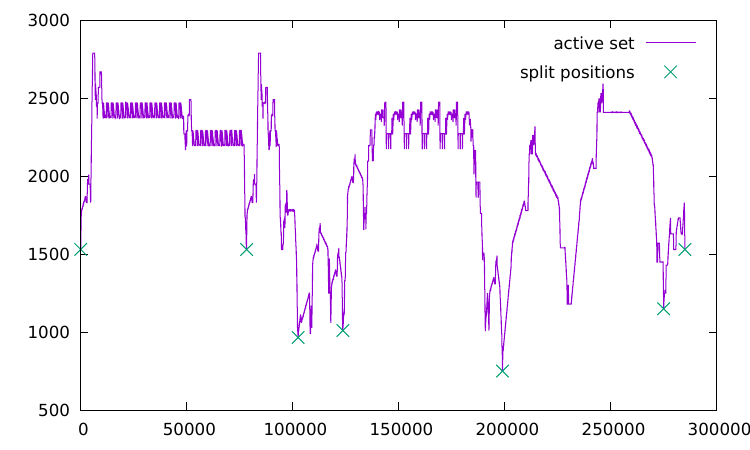}
  \label{fig:foam_ve}
	\caption{Tunnel flow simulation with OpenFOAM: Evolution of the 
	(vertex) cut size of the DAG over 
	potential split positions determined by the structure of target code.}
\end{figure}

Application of \verb!solve.exe! to the given 
partitioning of the DAG for a given conservative memory bound, for example,
$\overline{M}=284888,$ corresponding to the tape memory requirement of homogeneous adjoint mode, yields the following output:
\begin{lstlisting}[basicstyle=\footnotesize]
G_{1}=[ 1531 1531 78172 ]
G_{2}=[ 967 1531 24346 ]
G_{3}=[ 1011 967 21090 ]
G_{4}=[ 751 1011 75280 ]
G_{5}=[ 1151 751 75980 ]
G_{6}=[ 1531 1151 10020 ]

Dynamic Programming Table:
...
F'_{4,4}: MxG(0); fma=56535280; M=75280; 
F'_{4,3}: MxG(3); fma=72373870; M=96370; 
...
F'_{5,3}: GxM(4); fma=145846530; M=96370; 
F'_{5,2}: MxG(2); fma=173868776; M=120716; 
F'_{5,1}: MxG(1); fma=263844748; M=198888; 
...
F'_{6,1}: GxM(5); fma=279185368; M=198888; 

Optimal Cost=279185368
Memory Requirement=198888

Cost of
  homogeneous tangent mode=436163528
  homogeneous adjoint mode=436163528
  optimal preaccumulation=288747224+6690070967
                         =6978818191
\end{lstlisting}
Gaps in the output are marked by $\ldots$
The lines shown in the 
\lstinline[basicstyle=\footnotesize]{Dynamic Programming Table} yield the
six operations performed by the resulting non-trivial bracketing:
\begin{equation*}
  F'_{5,0} = \dot F_5 \cdot ((\dot F_4 \cdot (I_{m_3} \cdot \bar F_3 \cdot \bar F_2))\cdot \bar F_1 \cdot \bar F_0) \, .
\end{equation*}
Readers are encouraged to validate that
less conservative tape memory bounds increase the cost of Jacobian 
accumulation as follows:
$\overline{M}=198887\Rightarrow \fma=370959408,~M=102518$ and
$\overline{M}=102517\Rightarrow \fma=436163528,~M=0.$

\section{Conclusion and Outlook} \label{sec:concl}

This paper generalizes prior work on Jacobian chain products in the
context of AD. Tangents and adjoints of
differentiable subprograms of numerical simulation programs are typically available instead of the corresponding elemental Jacobian matrices. 
Dynamic programming makes
the underlying abstract combinatorial problem formulation computationally
tractable. 
Optimal combination of tangents and adjoints yield
sometimes impressive reductions of the overall operations count. 
Still, there are additional aspects to be considered for further evolution of
the general approach.

\paragraph{Exploitation of Jacobian Sparsity.}

Exploitation of sparsity of the $F_i$ in Equation~(\ref{eqn:jcp}) impacts the
computational cost estimate for their preaccumulation.
Various Jacobian compression
techniques based on coloring of different representations of the sparsity
patterns as graphs have been proposed for general \cite{Gebremedhin2005WCI}
as well as special \cite{Lulfesmann2014Seg} Jacobian sparsity patterns.
In general,
the minimization of
the overall computational cost becomes intractable as a consequence of
intractability of the underlying coloring problems. 

Exploitation of sparsity beyond the level of preaccumulation of the elemental
Jacobians in the context of Matrix-free Jacobian chaining yields further 
challenges to be addressed by future research. Consider, for example, the
simple scenario $F(\X)=F_2(F_1(\X))$ with diagonal $F'_2$ and some irregular
sparsity pattern of $F'_1$ yielding the same sparsity pattern for $F'.$
A single evaluation of the tangent $\dot{F}_2 \cdot \sum_{i=1}^n {\bf e}_i$
yields the optimally compressed $\hat{F}'_2$ as a vector of size $n,$ 
holding the diagonal of $F'_2.$ Let a compressed version $\hat{F}'_1$ of $F'_1$ 
be computed using one out of several powerful heuristics available 
\cite{Gebremedhin2005WCI}. 
Evaluation of $F'_2 \cdot F'_1$ as $\hat{F}'_2 \cdot \hat{F}'_1$ or
$\hat{F}'_2 \cdot \bar{F}'_1$ requires for $\hat{F}'_2$ to recompressed 
according to the sparsity pattern of $F'_1.$ Seamless chaining of compressed
Jacobians is prevented unless all compression is based on the sparsity
pattern of $F,$ which is likely to turn out suboptimal in general. 
The combination of matrix-free Jacobian chaining with the exploitation of
Jacobian sparsity remains the subject of ongoing research.

\paragraph{Exploitation of DAG Structure.}

Further exploitation of data dependence patterns through structural properties
of the concatenation of the elemental DAGs may lead to further decrease of the
computational cost. Vertex, edge, and face elimination techniques have been
proposed to allow for applications of the chain rule beyond Jacobian chain multiplication \cite{Naumann2004Oao}. For example, both bracketings of the sparse 
matrix chain product \\
$$
\footnotesize
        \begin{pmatrix}
        m^2_{0,0} & 0 & 0 \\
         0 & m^2_{1,1} & m^2_{1,2} \\
\end{pmatrix}
\begin{pmatrix}
        m^1_{0,0} & m^1_{0,1} & 0 \\
         0 & m^1_{1,1} & m^1_{1,2} \\
         0 & 0  & m^1_{2,2} \\
\end{pmatrix}
\begin{pmatrix}
        m^0_{0,0} &0  \\
         m^0_{1,0} & 0 \\
         0 & m^0_{2,1} 
\end{pmatrix}
$$ 
$\;$ \\
yield a computational
cost of $9 \fma.$ Full exploitation of distributivity enables computation
of the resulting matrix as \\
$$
\footnotesize
\begin{pmatrix}
        m^2_{0,0} (m^1_{0,0} m^0_{0,0} +m^1_{0,1} m^0_{1,0}) & 0 \\
        m^2_{1,1} m^1_{1,1} m^0_{1,0} & (m^2_{1,1} m^1_{1,2} + m^2_{1,2} m^1_{2,2}) m^0_{2,1}
\end{pmatrix}
$$
$\;$ \\
at the expense of only $8 \fma,$ which is easily translated into a 
vertex
elimination in the underlying DAG \cite{Griewank1991OtC}.
The combination of matrix-free Jacobian chaining with elimination techniques
remains the subject of ongoing research.

\paragraph{From Matrix Chains to Matrix DAGs.}
Let
$$
F=\begin{pmatrix}
	F_2 \circ F_1 \\
	F_3 \circ F_1
\end{pmatrix} : \R^2 \rightarrow \R^5
$$
such that $F_1 : \R^2 \rightarrow \R^4,$
$F_2 : \R^4 \rightarrow \R^1,$
$F_3 : \R^4 \rightarrow \R^4$
and $|E_i|=16$ for $i=1,2,3.$ Assuming availability of sufficient persistent
memory homogeneous adjoint mode turns out to be optimal
for $F_2 \circ F_1.$ The Jacobian of
$F_3 \circ F_1$ is optimally computed in homogeneous tangent mode which
yields a conflict for $F'_1.$ Separate optimization of the two Jacobian chain
products $F'_2 \cdot F'_1$ and $F'_3 \cdot F'_1$ yields a cumulative
computational cost of $1 \cdot (16+16) + 2 \cdot (16+16)=96 \fma.$
A better solution is
$$
F'=\begin{pmatrix}
(I_1 \cdot \bar{F}_2) \cdot (\dot{F}_1 \cdot I_2) \\
\dot{F}_3 \cdot (\dot{F}_1 \cdot I_2)
\end{pmatrix}
$$
yielding a slight decrease in the computational cost to
$2 \cdot 16 + 1 \cdot 16 + 1 \cdot 4 \cdot 8 + 2 \cdot 16=88 \fma.$
More significant savings can be expected for less simple DAGs.

The foundations for matrix-free elimination techniques were laid in
\cite{Naumann2023ETf}. Their combination with the findings of this paper
remains the subject of ongoing research.

\subsection*{Acknowledgement}
The experiment with OpenFOAM was conducted by Erik Schneidereit.

\end{document}